%% file: main.tex
\documentclass[runningheads]{llncs}
\usepackage{graphicx}
\usepackage{amsfonts}
\usepackage{hyperref}
\usepackage{url}
\usepackage{booktabs}
\usepackage{microtype}
\usepackage{comment}
\usepackage{lipsum}
\usepackage[ruled,vlined,linesnumbered]{algorithm2e}
\usepackage{mathtools}
\usepackage{xcolor}
\usepackage{array}
\usepackage{cleveref}
\usepackage{enumitem}
\usepackage{wrapfig}
\usepackage{booktabs, tabularx}
\usepackage{mathtools,trimclip,stackengine,scalerel}
\usepackage{eucal}
\usepackage{caption}
\usepackage{ifthen}
\usepackage{amsmath}
\usepackage{amssymb}
\usepackage{comment}
\usepackage{thmtools, thm-restate}
\usepackage{xspace}
\usepackage{tikz}
\usepackage{dsfont}
\usepackage{ifthen}
\newboolean{shortver}
\setboolean{shortver}{false}

\newcommand{\tool}{\texttt{Sphynx}\xspace}
\newcommand{\alg}{\texttt{PredHyperNet}\xspace}
\newcommand{\intalg}{\texttt{PredInt}\xspace}

\newcommand{\prop}{LDCP\xspace}
\newcommand{\propa}{locally differentially classification private (LDCP)\xspace}
\newcommand{\propl}{local differential classification privacy (LDCP)\xspace}
\newcommand{\propcap}{Local Differential Classification Privacy\xspace}

\begin{document}

\title{Verification of Neural Networks' \propcap}
\author{Roie Reshef \and Anan Kabaha \and Olga Seleznova \and Dana Drachsler-Cohen}
\authorrunning{R. Reshef et al.}
\institute{
Technion, Haifa, Israel
\email{\{sror@campus.,anan.kabaha@campus.,olga.s@,ddana@ee.\}technion.ac.il}
}

\maketitle
\input{abs}
\input{intro}
\input{prelim}
\input{probdef}
\input{split}
\input{alg}
\input{verify}
\input{eval}
\input{related}
\input{conc}

\subsubsection*{Acknowledgements}
We thank the anonymous reviewers for their feedback. 
      This research was supported by the Israel Science Foundation (grant No. 2605/20).

\bibliography{bib}
\ifthenelse{\boolean{shortver}}{}{
\appendix
\input{kde}
\input{proofs}
}
\end{document}

%% file: abs.tex
\begin{abstract}
Neural networks are susceptible to privacy attacks.
To date, no verifier can reason about the privacy of individuals participating in the training set.
We propose a new privacy property, called \emph{\propl}, extending local robustness to a differential privacy setting suitable for black-box classifiers.
Given a neighborhood of inputs, a classifier is \prop if it classifies all inputs the 
same regardless of whether it is trained with the full dataset or whether any single entry is omitted.
A naive algorithm is highly impractical because it involves training a very large number of networks and verifying local robustness of the given neighborhood separately for every network.
We propose \tool, an algorithm that computes an abstraction of all networks, with a high probability, from a small set of networks, and verifies \prop directly on the abstract network.
The challenge is twofold:
network parameters do not adhere to a known distribution probability, making it difficult to predict an abstraction, and predicting too large abstraction harms the verification.
Our key idea is to transform the parameters into a distribution given by KDE, allowing to keep the over-approximation error small.
To verify \prop, we extend a MILP verifier to analyze an abstract network.
Experimental results 
show that by training only 7\% of the networks, \tool predicts an abstract network 
obtaining $93\%$ verification accuracy and 
reducing the analysis time by $1.7\cdot10^4$x.
\end{abstract}

%% file: intro.tex
\section{Introduction}
\label{sec:intro}

Neural networks are successful in various tasks but are also vulnerable to attacks.
One kind of attacks that is gaining a lot of attention is privacy attacks.
Privacy attacks aim at revealing sensitive information about the network or its training set.
For example, membership inference attacks recover entries in the training set~\cite{member_infer,anal_mia,under_mia,enhance_mia,label_mia,member_label,Monerale23},
model inversion attacks reveal sensitive attributes of these entries~\cite{model_inv,FredriksonLJLPR14}, 
model extraction attacks recover the model's parameters~\cite{steal_model}, and property inference attacks infer global properties of the model~\cite{prop_ifer}.
Privacy attacks have been shown successful even against platforms providing a limited access to a model, including black-box access and a limited number of queries.
Such restricted access is common for platforms providing machine-learning-as-a-service (e.g., Google's Vertex AI\footnote{\url{https://cloud.google.com/vertex-ai}}, Amazon's ML on AWS\footnote{https://aws.amazon.com/machine-learning/}, and BigML\footnote{https://bigml.com/}).

\sloppy
A common approach to mitigate privacy attacks is \emph{differential privacy} (DP)~\cite{found_dp}.
DP has been adopted by numerous network training algorithms~\cite{dl_dp,ldp_dl,dp_grad,gauss_dp,dp_view}.
Given a privacy level, a DP training algorithm generates the same network with a similar probability, regardless of whether a particular individual's entry is included in the dataset.
However, DP is not an adequately suitable privacy criterion for black-box classifiers, for two reasons.
First, it poses a too strong requirement: it requires that the training algorithm returns the same network (i.e., assign the same score to every class), whereas black-box classifiers are considered the same if they predict the same class (i.e., assign the \emph{maximal} score to the same class).
Second, DP can
only be satisfied by randomized algorithms, adding noise to the computations.
Consequently, the accuracy of the resulting network decreases.
The amount of noise is often higher than necessary because the mathematical analysis of differentially private algorithms is highly challenging and thus practically not tight (e.g., it often relies on compositional theorems~\cite{found_dp}).
Thus, network designers often avoid adding noise to their networks.
This raises the question: \emph{What can a network designer provide as a privacy guarantee for individuals participating in the training set of a black-box classifier?}

We propose a new privacy property, called \emph{\propl}.
Our property is designed for black-box classifiers, whose training algorithm is not necessarily DP.
Conceptually, it extends the local robustness property, designed for adversarial example attacks~\cite{nn_prop,adver_explain}, to a ``deterministic differential privacy'' setting.
Local robustness requires that the network classifies all inputs in a given neighborhood the same.
We extend this property by requiring that the network {classifies} all inputs in a given neighborhood the same \emph{regardless of whether a particular individual's entry is included in the dataset}.

Proving that a network is \prop is challenging, because it requires to check the local robustness of a very large number of networks: $|\mathcal{D}|+1$ networks, where $\mathcal{D}$ is the dataset, which is often large (e.g., $>10$k entries).
To date,
verification of local robustness~\cite{smt_relu,safe_abst,robust_milp,abst_cert,Dung20,bound_prop,poly_gpu,poly_rnn}, analyzing a single network, takes non-negligible time.
A naive, accurate but highly unscalable algorithm checks \prop by training every possible network (i.e., one for every possibility to omit a single entry from the dataset), checking that all inputs in the neighborhood are classified the same network-by-network, and verifying that all networks classify these inputs the same.
However, this naive algorithm does not scale since training and analyzing thousands of networks is highly time-consuming.

 \begin{figure*}[t]
    \centering
  \includegraphics[width=1\linewidth, trim=0 170 0 0, clip, page=2]{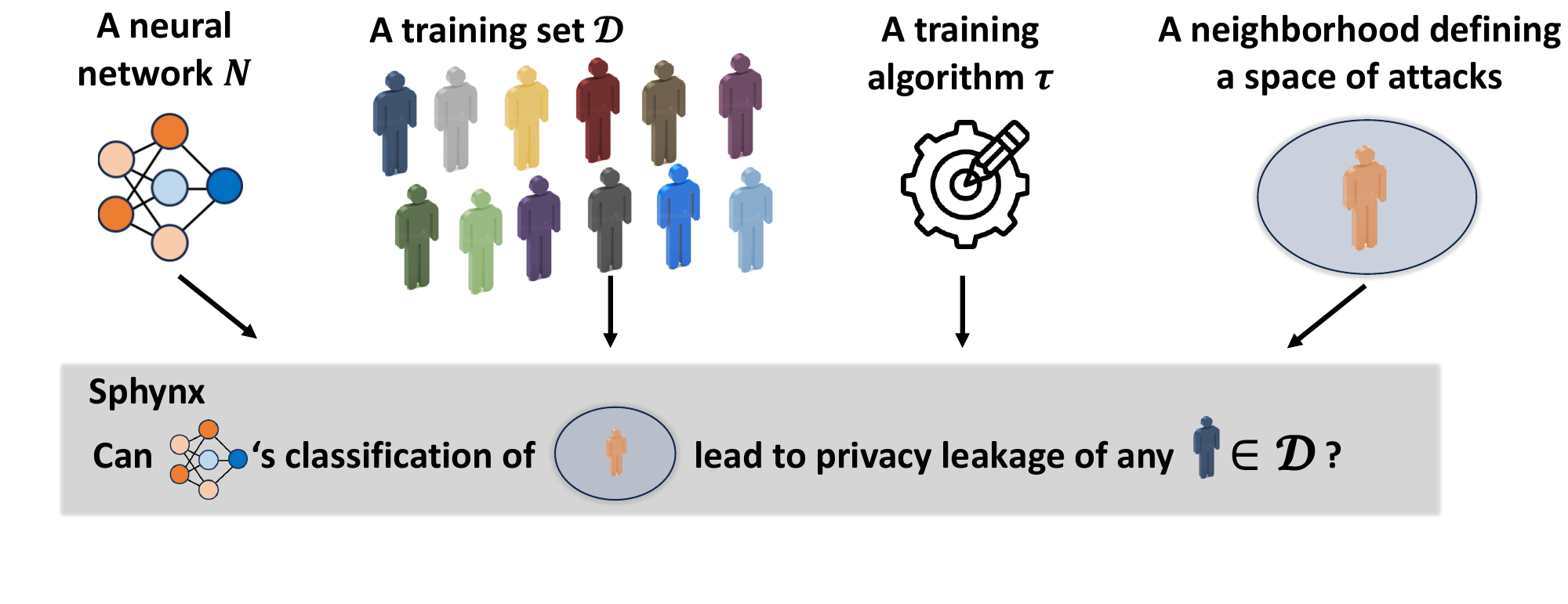}
    \caption{\tool checks leakage of individuals' entries at a given neighborhood.}
    \label{fig::sphynx_intro}
\end{figure*}

We propose \tool ({\bf S}afety {\bf P}rivacy analyzer via {\bf Hy}per-{\bf N}etworks) for determining whether a network is \prop at a given neighborhood (\Cref{fig::sphynx_intro}).
\tool takes as inputs a network, its training set, its training algorithm, and a neighborhood.
Instead of training $|\mathcal{D}|$ more networks, \tool computes a \emph{hyper-network} abstracting these networks with a high probability (under several conditions).
A hyper-network abstracts a set of networks by associating to each network parameter an interval that contains the respective values in all networks.
If \tool would train all networks, computing the hyper-network would be straightforward.
However, \tool's goal is to reduce the high time overhead and thus it does not train all networks.
Instead, it predicts a hyper-network from a small number of networks.
\tool then checks \prop at the given neighborhood directly on the hyper-network.
These two ideas enable \tool to obtain a practical analysis time.
The main challenges in predicting a hyper-network are: (1)~network parameters do not adhere to a known probability distribution and (2)~the inherent trade-off between a sound abstraction, where each parameter's interval covers all values of this parameter in every network (seen and unseen) and the ability to succeed in verifying \prop given the hyper-network.
Naturally, to obtain a sound abstraction, it is best to consider large intervals for each network parameter, e.g., by adding noise to the intervals (like in adaptive data analysis~\cite{valid_adapt,prevent_false,stab_adapt,general_adapt,pres_adapt,adapt_median,emp_adapt,limit_general}).
However, the larger the intervals the harder it is to verify \prop, because the hyper-network abstracts many more (irrelevant) networks.
To cope, \tool transforms the network parameters into a distribution given by kernel density estimation (KDE), allowing to predict the intervals without adding noise.

To predict a hyper-network, \tool executes an iterative algorithm.
In every iteration, it samples a few entries from the dataset.
For each entry, it trains a network given all the dataset except this entry.
Then, given all trained networks, \tool predicts a hyper-network, i.e., it computes an interval for every network parameter.
An interval is computed by transforming every parameter's observed values into a distribution given by KDE, using normalization and the Yeo-Johnson transformation~\cite{transform}.
Then, \tool estimates whether the hyper-network abstracts every network with a high probability, and if so, it terminates.

Given a hyper-network, \tool checks \prop directly on the hyper-network.
To this end, we extend a local robustness verifier~\cite{robust_milp}, relying on mixed-integer linear programming (MILP), to analyze a hyper-network.
Our extension replaces the equality constraints, capturing the network's affine computations, with inequality constraints, since 
our network parameters are associated with intervals and not real numbers.
To mitigate an over-approximation error, we propose two approaches.
The first approach relies on preprocessing to the network's inputs and the second one relies on having a lower bound on the inputs.

We evaluate \tool on data-sensitive datasets: Adult Census~\cite{uci}, Bank Marketing~\cite{bank_market} and Default of Credit Card Clients~\cite{credit_card}.
We verify \prop on three kinds of neighborhoods for checking safety 
to label-only membership attacks~\cite{label_mia,member_label,Monerale23}, adversarial example attacks in a DP setting (like~\cite{LecuyerAG0J19,PhanTHJSD20}), and sensitive attributes (like~\cite{model_inv,FredriksonLJLPR14}).
We show that by training only 7\% of the networks, \tool predicts a hyper-network abstracting an (unseen) network with a probability of at least $0.9$.
Our hyper-networks obtain 93\% verification accuracy.
Compared to the naive algorithm, \tool provides a significant speedup: it reduces the training time by 13.6x and the verification time by $1.7\cdot 10^4$x.

%% file: prelim.tex
\section{Preliminaries}
\label{sec:prelim}

In this section, we provide the necessary background.

\paragraph{Neural network classifiers}
We focus on binary classifiers, which are popular for data-sensitive tasks.
As example, we describe the data-sensitive datasets used in our evaluation and their classifier's task (\Cref{sec:eval}).
Adult Census~\cite{uci} consists of user records of the socioeconomic status of people in the US.
The goal of the classifier is to predict whether a person's yearly income is higher or lower than 50K USD.
Bank Marketing~\cite{bank_market} consists of user records of direct marketing campaigns of a Portuguese banking institution.
The goal of the classifier is to predict whether the client will subscribe to the product or not.
Default of Credit Card Clients~\cite{credit_card} consists of user records of demographic factors, credit data, history of payments, and bill statements of credit card clients in Taiwan.
The goal of the classifier is to predict whether the default payment will be paid in the next month.
We note that our definitions and algorithms easily extend to non-binary classifiers.
A binary classifier $N$ maps an input, a user record, $x\in\mathcal{X}\subseteq[0,1]^d$ to a real number $N(x)\in\mathbb{R}$.
If $N(x)\geq0$, we say the classification of $x$ is $1$ and write $\text{class}(N(x))=1$, otherwise, it is $-1$, i.e., $\text{class}(N(x))=-1$.
We focus on classifiers implemented by a fully-connected neural network.
This network consists of an input layer followed by $L$ layers.
The input layer $x_0$ takes as input $x\in\mathcal{X}$ and passes it as is to the next layer (i.e., $x_{0,k}=x_k$).
The next layers are functions, denoted $f_1,f_2,\dots,f_L$, each taking as input the output of the preceding layer.
The network's function is the composition of the layers: $N(x)=f_L(f_{L-1}(\cdots(f_1(x))))$.
A layer $m$ consists of neurons, denoted $x_{m,1},\ldots,x_{m,k_m}$.
Each neuron takes as input the outputs of all neurons in the preceding layer and outputs a real number.
The output of layer $m$ is the vector $(x_{m,1},\ldots,x_{m,k_m})^T$ consisting of all its neurons' outputs.
A neuron $x_{m,k}$ has a weight for each input $w_{m,k,k'}\in \mathbb{R}$ and a bias $b_{m,k}\in \mathbb{R}$.
Its function is the composition of an affine computation, $\hat{x}_{m,k}=b_{m,k}+\sum_{k'=1}^{k_{m-1}}w_{m,k,k'}\cdot x_{m-1,k'}$, followed by an activation function computation, $x_{m,k}=\sigma(\hat{x}_{m,k})$.
Activation functions are typically non-linear functions.
In this work, we focus on the ReLU activation function, $\text{ReLU}(\hat{x})=\max(0,\hat{x})$.
We note that, while we focus on fully-connected networks, our approach can extend to other architectures, e.g., convolutional networks or residual networks.
The weights and biases of a neural network are determined by a training process.
A training algorithm $\mathcal{T}$ takes as inputs a network (typically, with random weights and biases) and a labelled training set $\mathcal{D}=\{(x_1,y_1),\ldots,(x_n,y_n)\}\subseteq \mathcal{X}\times\{-1,+1\}$.
It returns a network with updated weights and biases.
These parameters are computed with the goal of minimizing a given loss function, e.g., binary cross-entropy, capturing the level of inaccuracy of the network.
The computation typically relies on iterative numerical optimization, e.g., stochastic gradient descent (SGD).

\paragraph{Differential privacy (DP)}
DP focuses on algorithms defined over arrays (in our context, a dataset).
At high-level, an algorithm is DP if for any two inputs differing in a single entry, it returns the same output with a similar probability.
Formally, DP is a probabilistic privacy property requiring that the probability of returning different outputs is upper bounded by an expression defined by two parameters, denoted $\epsilon$ and $\delta$~\cite{found_dp}.
Note that this requirement is too strong for classifiers providing only black-box access, which return only the input's classification.
For such classifiers, it is sufficient to require that the classification is the same, and there is no need for the network's output (the score of every class) to be the same.
To obtain the DP guarantee, DP algorithms add noise, drawn from some probability distribution (e.g., Laplace or Gaussian), to the input or their computations.
That is, DP algorithms trade off their output's accuracy with privacy guarantees: the smaller the DP parameters (i.e., $\epsilon$ and $\delta$) the more private the algorithm is, but its outputs are less accurate.
The accuracy loss is especially severe in DP algorithms that involve loops in which every iteration adds noise.
The loss is high because (1)~many noise terms are added and (2)~the mathematical analysis is not tight (it typically relies on compositional theorems~\cite{found_dp}), leading to adding a higher amount of noise than necessary to meet the target privacy guarantee.
Nevertheless, DP has been adopted by numerous network training algorithms~\cite{dl_dp,ldp_dl,dp_grad,gauss_dp,dp_view}.
For example, one algorithm adds noise to every gradient computed during training~\cite{dl_dp}.
Consequently, the network's accuracy decreases significantly, discouraging network designers from employing DP.
To cope, we propose a (non-probabilistic) privacy property that (1) only requires the network's classification to be the same and (2) can be checked even if the network has not been trained by a DP training algorithm.

\paragraph{Local robustness}
Local robustness has been introduced in response to adversarial example attacks~\cite{nn_prop,adver_explain}.
In the context of network classifiers, an adversarial example attack is given an input and a space of perturbations and it returns a perturbed input that causes the network to misclassify.
Ideally, to prove that a network is robust to adversarial attacks, one should prove that \emph{for any valid input}, the network classifies the same under any valid perturbation.
In practice, the safety property that has been widely-studied is \emph{local robustness}.
 A network is locally robust at \emph{a given input} if perturbing the input by a perturbation in a given space does not cause the network to change the classification.
 Formally, the space of allowed perturbations is captured by a \emph{neighborhood} around the input.
\begin{definition}[Local Robustness]
Given a network $N$, an input $x\in\mathcal{X}$, a neighborhood $I(x)\subseteq \mathcal{X}$ containing $x$, and a label $y\in\{-1,+1\}$, the network $N$ is \emph{locally robust at $I(x)$ with respect to $y$} if $\forall x'\in I(x).\ \text{class}(N(x'))=y$.
\end{definition}
A well-studied definition of a neighborhood is the \emph{$\epsilon$-ball} with respect to the $L_\infty$ norm~\cite{smt_relu,safe_abst,robust_milp,abst_cert,Dung20,bound_prop,poly_gpu,poly_rnn}.
Formally, given an input $x$ and a bound on the perturbation amount $\epsilon\in\mathbb{R}^+$, the $\epsilon$-ball is: 
$I^\infty_\epsilon(x)=\left\{x'\mid\|x'-x\|_\infty \leq \epsilon\right\}$.
A different kind of a neighborhood captures \emph{fairness} with respect to a given sensitive feature 
$i\in [d]$  (e.g., gender)~\cite{Bastani0S19,redu_fair,paral_fair,RuossBFV20}:
$I^S_i(x)=\left\{x'\mid\bigwedge_{j\in [d]\setminus\{i\}} x'_j=x_j\right\}$.

%% file: probdef.tex
\section{\propcap}
\label{sec:probdef}
In this section, we define the problem of verifying that a network is \propa at a given neighborhood.

\paragraph{Local differential classification privacy (LDCP)}
Our property is defined given a classifier $N$, trained on a dataset $\mathcal{D}$, and a neighborhood $I$, defining a space of attacks (perturbations).
It considers $N$ private with respect to an individual's entry $(x',y')$ if $N$ classifies all inputs in $I$ the same, whether $N$ has been trained with $(x',y')$ or not.
If there is discrepancy in the classification, the attacker may exploit it to infer information about $x'$.
Namely, if the network designer cared about a single individual's entry $(x',y')$, our privacy property would be defined over two networks: the network trained with $(x',y')$ and the network trained without $(x',y')$.
Naturally, the network designer wishes to show that the classifier is private for every $(x',y')$ participating in $\mathcal{D}$.
Namely, our property is defined over $|\mathcal{D}|+1$ networks.
The main challenge in verifying this privacy property is that the training set size is typically very large ($>$10k entries).
Formally, our property is 
an extension of local robustness to a differential privacy setting suitable for classifiers providing black-box access.
It requires that the inputs in $I$ are classified the same by \emph{every} network trained on $\mathcal{D}$ or trained on $\mathcal{D}$ except for any single entry.
Unlike DP, our definition is applicable to any network, even those trained without a probabilistic noise.
We next formally define our property.
\begin{definition}[\propcap]\label{def:ldcp}
Given a network $N$ trained on a dataset $\mathcal{D}\subseteq\mathcal{X}\times\{-1,+1\}$ using a training algorithm $\mathcal{T}$, a neighborhood $I(x)\subseteq\mathcal{X}$ and a label $y\in\{-1,+1\}$, the network $N$ is \emph{\propa} if (1)~$N$ is locally robust at $I(x)$ with respect to $y$, and (2)~for every $(x',y')\in\mathcal{D}$, the network of the same architecture as $N$ trained on $\mathcal{D}\setminus\{(x',y')\}$  using $\mathcal{T}$ is locally robust at $I(x)$ with respect to $y$.
\end{definition}
Our privacy property enables to prove safety against privacy attacks by defining a suitable set of neighborhoods.
For example, several membership inference attacks~\cite{label_mia,member_label,Monerale23} are label-only attacks.
Namely, they assume the attacker has a set of inputs and they can query the network to obtain their classification.
To prove safety against these attacks, given a set of inputs $X\subseteq \mathcal{X}$, one has to check \prop for every neighborhood defined by $I^\infty_0(x)$ where $x\in X$.
To prove safety against attacks aiming to reveal sensitive features (like~\cite{model_inv,FredriksonLJLPR14}), given a set of inputs $X\subseteq\mathcal{X}$, one has to check that the network classifies the same regardless of the value of the sensitive feature.
This can be checked by checking \prop for every neighborhood defined by $I^S_i(x)$ where $x\in X$ and $i$ is the sensitive feature.


\paragraph{Problem definition}
We address the problem of determining whether a network is \propa at a given neighborhood, while minimizing the analysis time.
A definite answer involves training and verifying $|\mathcal{D}|+1$ networks (the network trained given all entries in the training set, and for each entry in the training set, the network trained given all entries except this entry).
However, this results in a very long training time and verification time (even local robustness verification takes a non-negligible time).
Instead, our problem is to provide an answer which is correct with a high probability.
This problem is highly challenging.
On the one hand, the fewer trained networks the lower the training and verification time.
On the other hand, determining an answer, with a high probability, from a small number of networks is challenging, since network parameters do not adhere to a known probabilistic distribution.
Thus, our problem is not naturally amenable to a probabilistic analysis.

\paragraph{Prior work}
Prior work has considered different aspects of our problem.
As mentioned, several works adapt training algorithms to guarantee that the resulting network satisfies differential privacy~\cite{dl_dp,ldp_dl,dp_grad,gauss_dp,dp_view}.
However, these training algorithms tend to return networks of lower accuracy.
A different line of research proposes algorithms for \emph{machine unlearning}~\cite{unlearn_mask,eval_unlearn,toward_unlearn}, in which an algorithm retrains a network ``to forget'' some entries of its training set.
However, these approaches do not guarantee that the forgetting network is equivalent to the network obtained by training without these entries from the beginning.
It is also more suitable for settings in which there are multiple entries to forget, unlike our differential privacy setting which focuses on omitting a single entry at a time.
Several works propose verifiers to determine local robustness~\cite{verify_backdoor,verify_fair,verify_pw,smt_verify,paral_verify,convex_verify,smt_relu}.
However, these analyze a single network and not a group of similar networks.
A recent work proposes a proof transfer between similar networks~\cite{proof_trans} in order to reduce the verification time of similar networks.
However, this work requires to explicitly have all networks, which is highly time consuming in our setting.
Other works propose verifiers that check robustness to data poisoning or data bias~\cite{Sun0GP22,MeyerAD21,DrewsAD20}. These works consider an attacker that can manipulate or omit up to several entries from the dataset, similarly to \prop allowing to omit up to a single entry. However, these verifiers either target patch attacks of image classifiers~\cite{Sun0GP22}, which allows them to prove robustness without considering every possible network, or target decision trees~\cite{MeyerAD21,DrewsAD20}, relying on predicates, which are unavailable for neural networks. Thus, neither is applicable to our setting. 

%% file: split.tex
\section{Our Approach}
\label{sec:split}

In this section, we present our approach for determining whether a network is \propa.
Our key idea is to \emph{predict an abstraction} of all concrete networks from a small number of networks.
We call the abstraction a \emph{hyper-network}.
Thanks to this abstraction, \tool does not need to train all concrete networks and neither verify multiple concrete networks.
Instead, \tool first predicts an abstract hyper-network, given a network $N$, its training set $\mathcal{D}$ and its training algorithm $\mathcal{T}$, and then verifies \prop directly on the hyper-network.
This is summarized in~\Cref{fig:flow}.
We next define these terms.

\begin{figure}[t]
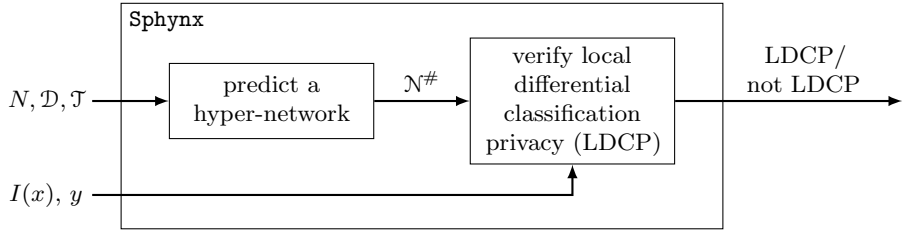

    \centering
    \include{flow}
    \caption{Given a network classifier $N$, its training set $\mathcal{D}$ and its training algorithm $\mathcal{T}$, \tool predicts an abstract hyper-network $\mathcal{N}^\#$.
    It then checks whether $\mathcal{N}^\#$ is locally robust at $I(x)$ with respect to $y$ to determine whether $N$ is \prop.}
    \label{fig:flow}
\end{figure}

\paragraph{Hyper-networks}
A \emph{hyper-network} abstracts a set of networks $\mathcal{N}=\{N_1,\ldots,N_K\}$ with the same architecture and in particular the same set of parameters, i.e., weights $\mathcal{W}=\{w_{1,1,1},\ldots,w_{L,d_L,d_{L-1}}\}$ and biases $\mathcal{B}=\{b_{1,1},\ldots,b_{L,d_L}\}$.
The difference between the networks is the parameters' values.
In our context, $\mathcal{N}$ consists of the network trained with the full dataset and the networks trained without any single entry:   $\mathcal{N}=\{\mathcal{T}(N,\mathcal{D})\}\cup \{\mathcal{T}(N,\mathcal{D}\setminus\{(x',y')\})\mid (x',y')\in \mathcal{D}\}$.
A \emph{hyper-network} is a network $\mathcal{N}^\#$ with the same architecture and set of parameters, but the domain of the parameters is not $\mathbb{R}$ but rather an abstract domain $\mathcal{A}$.
As standard, we assume the abstract domain corresponds to a lattice and is equipped with a concretization function $\gamma$.
We focus on a non-relational abstraction, where each parameter is abstracted independently.
The motivation is twofold.
First, non-relational domains are computationally lighter than relational domains.
Second, the relation between the parameters is highly complex because it depends on a long series of optimization steps (e.g., SGD steps).
Thus, while it is possible to bound these computations using simpler functions (e.g., polynomial functions), the over-approximation error would be too large to be useful in practice.

Formally, given a set of networks $\mathcal{N}$ with the same architecture and set of parameters
$\mathcal{W}\cup \mathcal{B}$ and given an abstract domain $\mathcal{A}$ and a concretization function $\gamma: \mathcal{A}\rightarrow \mathcal{P}(\mathbb{R})$, a \emph{hyper-network} is a network $\mathcal{N}^\#$ with the same architecture and set of parameters, where the parameters range over $\mathcal{A}$ and satisfy the following:
\begin{equation*}
    \forall N'\in\mathcal{N} : \left[\forall w_{m,k,k'}\in \mathcal{W} : w^{N'}_{m,k,k'}\in\gamma \left(w_{m,k,k'}^\#\right) \wedge \forall b_{m,k}\in \mathcal{B} : b^{N'}_{m,k}\in \gamma \left(b_{m,k}^\#\right)\right]
\end{equation*}
where $w^{N'}_{m,k,k'}$ and $b^{N'}_{m,k}$ are the values of the parameters in the network $N'$ and
$w^\#_{m,k,k'}$ and $b^\#_{m,k}$ are the values of the parameters in the hyper-network $\mathcal{N}^\#$.


 \begin{figure*}[t]
    \centering
  \includegraphics[width=1\linewidth, trim=0 200 0 0, clip, page=3]{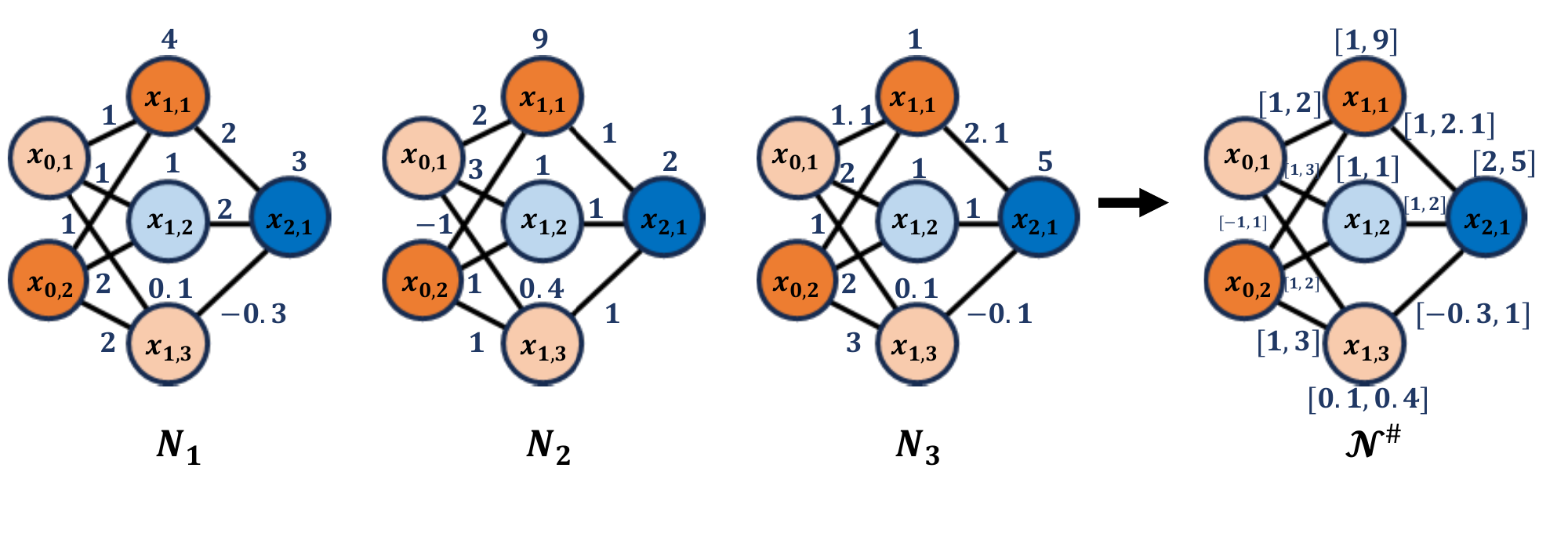}
    \caption{Three networks, $N_1, N_2, N_3$, and their interval hyper-network $\mathcal{N}^\#$.}
    \label{fig:sample}
\end{figure*}

\paragraph{Interval abstraction}
In this work, we focus on the interval domain.
Namely, the abstract elements are intervals $[l,u]$, with the standard meaning: an interval $[l,u]$ abstracts all real numbers between $l$ and $u$.
We thus call our hyper-networks \emph{interval hyper-networks}.
\Cref{fig:sample} shows an example of an interval hyper-network.
An interval corresponding to a weight is shown next to the respective edge, and an interval corresponding to a bias is shown next to the respective neuron.
For example, the neuron $x_{1,2}$ has two weights and a bias whose values are:
$w^\#_{1,2,1}=[1,3]$, $w^\#_{1,2,2}=[1,2]$, and $b^\#_{1,2}=[1,1]$.
Computing an interval hyper-network is straightforward if all concrete networks are known.
However, computing all concrete networks defeats the purpose of having a hyper-network.
Instead, \tool predicts an interval hyper-network with a high probability (\Cref{sec:alg}).

\paragraph{Checking \prop given a hyper-network}
Given an interval hyper-network $\mathcal{N}^\#$ for a network $N$, a neighborhood $I(x)$ and a label $y$, \tool checks whether $N$ is \prop by checking whether $\mathcal{N}^\#$ is locally robust at $I(x)$ with respect to $y$.
If $\mathcal{N}^\#$ is robust, then $N$ is \prop at $I(x)$, with a high probability.
Otherwise, \tool determines that $N$ is not \prop at $I(x)$.
Note that this is a conservative answer since $N$ is either not \prop or that the abstraction or the verification lose too much precision.
\tool verifies local robustness of $\mathcal{N}^\#$ by extending a MILP verifier~\cite{robust_milp} checking local robustness of a neural network (\Cref{sec:verify}).

%% file: flow.tex
\tikzstyle{io} = [rectangle, minimum width=1cm, minimum height=1cm, text centered, text width=1cm, draw=white, fill=white]
\tikzstyle{empty} = [rectangle, minimum width=1cm, minimum height=1cm, text centered, text width=1cm, draw=white, fill=white]
\tikzstyle{block} = [rectangle, minimum width=2.5cm, minimum height=1cm, text centered, text width=2.5cm, draw=black, fill=white]
\tikzstyle{box} = [rectangle, minimum width=8cm, minimum height=3cm, text centered, draw=black, fill=white]
\tikzstyle{arrow} = [thick,->,>=latex]

\begin{tikzpicture}[node distance=1cm]

\node (all) [box] {};
\node[below right] at (all.north west) {\tool};
\node (data) [io, left of=all, xshift=-4cm, yshift=0.2cm] {$N,\mathcal{D},\mathcal{T}$};
\node (gen) [block, right of=data, xshift=2cm] {predict a hyper-network};
\node (verify) [block, right of=gen, xshift=3cm] {verify \propl};
\node (prop) [io, below of=data, yshift=-0.25cm] {$I(x)$, $y$};
\node (res) [empty, right of=verify, xshift=4cm] {};

\draw [arrow] (data) -- (gen);
\draw [arrow] (gen) -- node[anchor=south] {$\mathcal{N}^\#$} (verify);
\draw [arrow] (prop) -| (verify);
\draw [arrow] (verify) -- node[anchor=south, text width=2cm, text centered, xshift=0.2cm]{\prop/\\not \prop} (res);

\end{tikzpicture} 

%% file: alg.tex
\section{Sphynx: Safety Privacy Analyzer via Hyper-Networks}
In this section, we present \tool, our system for verifying \propl.
As described, 
it consists of two components, the first component predicts a hyper-network, while the second one verifies \prop.

\subsection{Prediction of an Interval Hyper-Network}
\label{sec:alg}
In this section, we introduce \tool's algorithm for predicting an interval hyper-network, called \alg.
\alg takes as inputs a network $N$, its training set $\mathcal{D}$, its training algorithm $\mathcal{T}$, and a probability error bound $\alpha$, where $\alpha \in (0,1)$ is a small number.
It returns an interval hyper-network $\mathcal{N}^\#$ which, with probability $1-\alpha$ (under certain conditions), abstracts an (unseen) concrete network returned by $\mathcal{T}$ given $N$ and $\mathcal{D}\setminus\{(x,y)\}$, where $(x,y)\in\mathcal{D}$.
The main idea is to predict an abstraction for every network parameter from a small number of concrete networks.
To minimize the number of networks,
\alg executes iterations.
An iteration trains $k$ networks and predicts an interval hyper-network using all trained networks.
If the intervals' distributions have not converged to the expected distributions, another iteration begins.
We next 
provide details.

\paragraph{\alg's algorithm}
\alg (\Cref{alg:gen}) begins by initializing the set of trained networks \texttt{nets} to $N$ and the set of entries \texttt{entr}, whose corresponding networks are in \texttt{nets}, to $\emptyset$.
It initializes the previous hyper-network $\mathcal{N}^\#_{prev}$ to $\bot$ and the current hyper-network $\mathcal{N}^\#_{curr}$ to the interval abstraction of $N$, i.e., the interval of a network parameter $w\in \mathcal{W}\cup\mathcal{B}$ (a weight or a bias) is $[w^N,w^N]$.
Then, while the stopping condition~(\Cref{ln:stop}), checking convergence using the Jaccard distance as described later, is \texttt{false}, it runs an iteration.
An iteration trains $k$ new networks~(\Cref{ln:trbegin}--\Cref{ln:trend}).
A training iteration samples an entry $(x,y)$, adds it to \texttt{entr}, and runs the training algorithm $\mathcal{T}$ on (the architecture of) $N$ and $\mathcal{D}\setminus\{(x,y)\}$.
The trained network is added to \texttt{nets}.
\alg then computes an interval hyper-network from \texttt{nets}~(\Cref{ln:abbegin}--\Cref{ln:abend}).
The computation is executed via \intalg (\Cref{ln:abend}) independently on each network parameter $w$.
\intalg's inputs are all observed values of $w$ and a probability error bound $\alpha'$, which is $\alpha$ divided by the number of 
network parameters.

\begin{algorithm}[t]
\caption{\alg($N$, $\mathcal{D}$, $\mathcal{T}$, $\alpha$)}
\label{alg:gen}
\DontPrintSemicolon
\KwIn{a network $N$, a training set $\mathcal{D}$, a training algorithm $\mathcal{T}$, an error bound $\alpha$.}
\KwOut{an interval hyper-network.}
$\text{nets}\gets\{N\}$\;
$\text{entr}\gets\emptyset$\;
$\mathcal{N}^\#_{prev}\gets\bot$\;
$\mathcal{N}^\#_{curr}\gets N^\#$\;
\While{$\Sigma_{w\in \mathcal{W}\cup\mathcal{B}} \mathds{1}\left[1-J([l_{curr}^w,u_{curr}^w],[l_{prev}^w,u_{prev}^w])\leq R\right] < M\cdot |\mathcal{W}\cup\mathcal{B}|$}{\label{ln:stop}
$\mathcal{N}^\#_{prev}\gets \mathcal{N}^\#_{curr}$\;
\For{$k$ iterations}
{\label{ln:trbegin}
$(x,y)\gets\text{Random}\left(\mathcal{D}\setminus\text{entr}\right)$\;
$\text{entr}\gets\text{entr}\cup\{(x,y)\}$\;
$\text{nets}\gets\text{nets}\cup\left\{\mathcal{T}\left(N,\mathcal{D}\setminus\{(x,y)\}\right)\right\}$\;
}\label{ln:trend}
\For{$w\in\mathcal{W}\cup\mathcal{B}$}
{\label{ln:abbegin}
$\mathcal{V}_w\gets\{w^{N'} \mid N'\in\text{nets}\}$\;
$w^{\mathcal{N}^\#_{curr}}\gets\intalg\left(\mathcal{V}_w,\frac{\alpha}{|\mathcal{W}\cup\mathcal{B}|}\right)$\;
}
}\label{ln:abend}
\Return $\mathcal{N}^\#_{curr}$\;
\end{algorithm}

\paragraph{Interval prediction: overview}
\label{sec:algpredict}
\intalg predicts an interval for a parameter $w$ from a (small) set of values $\mathcal{V}_w=\{w_1,\ldots,w_{K}\}$, obtained from the concrete networks, and given an error bound $\alpha'$.
If it used interval abstraction, $w$ would be mapped to the interval defined by the minimum and maximum in $\mathcal{V}_w$.
However, this approach cannot guarantee to abstract the unseen concrete networks, since we aim to rely on a number of concrete networks significantly smaller than $|\mathcal{D}|$.
Instead, \intalg defines an interval by predicting the minimum and maximum of $w$, over all its values in the (seen and unseen) networks.
There has been an extensive work on estimating statistics, with a high probability, of an unknown probability distribution (e.g., expectation~\cite{valid_adapt,prevent_false,stab_adapt,general_adapt,pres_adapt,adapt_median,emp_adapt,limit_general}).
However, \intalg's goal is to estimate the \emph{minimum} and \emph{maximum} of unknown samples from an unknown probability distribution.
The challenge is that, unlike other statistics, the minimum and maximum are highly sensitive to outliers.
To cope, \intalg transforms the unknown probability distribution of $w$ into a known one and then predicts the minimum and maximum.
Before the transformation, \intalg normalizes the values in $\mathcal{V}_w$.
Overall, \intalg's operation is (\Cref{fig:alg_flow}):
(1)~it normalizes the values in $\mathcal{V}_w$, (2)~it transforms the values into a known probability distribution,
(3)~it predicts the minimum and maximum by computing a confidence interval with a probability of $1-\alpha'$, and (4)~it inverses the transformation and normalization to fit the original scale.
We note that executing normalization and transformation and then their inversion does not result in any information loss, because they are bijective functions.
We next provide details.

\begin{figure}[t]
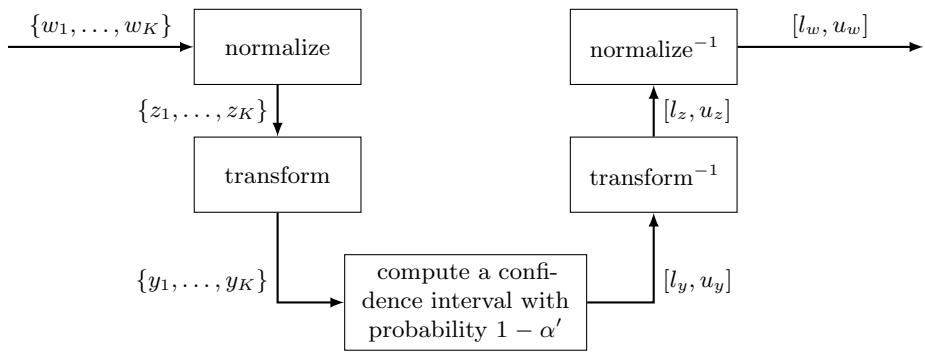

    \centering
    \include{alg_flow}
    \caption{The flow of \intalg for predicting an interval for a network parameter.}
    \label{fig:alg_flow}
\end{figure}

\paragraph{Transformation and normalization}
\intalg transforms the values in $\mathcal{V}_w$ to make them seem as if drawn from a known probability distribution.
It employs the Yeo-Johnson transformation~\cite{transform}, transforming an unknown distribution into a Gaussian distribution.
This transformation has the flexibility that the input random variables can have any real value.
It has a parameter $\lambda$, whose value is determined using maximum likelihood estimation (MLE) (i.e., $\lambda$ maximizes the likelihood that the transformed data is Gaussian).
It is defined as follows:
\[
    T_\lambda(z) =
\left\{
\begin{array}{lrlr}
    \frac{(1+z)^\lambda-1}{\lambda}, & \lambda\neq0,z\geq0; \phantom{check111} &
    \log(1+z), & \lambda=0,z\geq0 \\
    -\frac{(1-z)^{2-\lambda}-1}{2-\lambda}, & \phantom{11} \lambda\neq2,z<0; \phantom{check111}&
    -\log(1-z), & \phantom{11}  \lambda=2,z<0
    \end{array}
    \right\}
\]
\intalg employs several adaptations to this transformation to fit our setting better.
First, it requires $\lambda\in [0,2]$.
This is because if $\lambda<0$, the transformed outputs have an upper bound $\frac{1}{-\lambda}$
and if $\lambda>2$, they have a lower bound $-\frac{1}{\lambda-2}$.
Since our goal is to predict the minimum and maximum values, we do not want the transformed outputs to artificially limit them.
Second, the Yeo-Johnson transformation transforms into a Gaussian distribution.
Instead, we transform to a distribution given by kernel density estimation (KDE) with a Laplace kernel, which is better suited for our setting~(\Cref{sec:eval}).
We describe KDE in
\ifthenelse{\boolean{shortver}}{the extended version~\cite{}.}{\Cref{app:kde}.}
Third, the Yeo-Johnson transformation operates best when given values centered around zero.
However, training algorithms produce parameters of different scales and centered around different values.
Thus, before the transformation, \intalg normalizes the values in $\mathcal{V}_w$ to be centered around zero and have a similar scale as follows: $z_i\gets\frac{w_i-\mu}{\Delta}$, $\forall i\in [K]$.
Consequently, \intalg is invariant to translation and scaling and is thus more robust to the computations of the training algorithm~$\mathcal{T}$.
There are several possibilities to define the normalization's parameters, $\mu$ and $\Delta$, each is based on a different norm, e.g., the $L_1$, $L_2$ or $L_\infty$ norm.
In our setting, a normalization based on the $L_1$ norm works best, since we use a Laplace kernel.
Its center point is $\mu=\text{median}(w_1,\ldots,w_K)$ and its scale is the centered absolute first moment: $\Delta=\frac{1}{K}\sum_{i=1}^K|w_i-\mu|$.

\paragraph{Interval prediction}
After the transformation, \intalg has a cumulative distribution function (CDF) for the transformed values: $F_{y}(v)=\mathbb{P}\{y\leq v\}$.
Given the CDF, we compute a confidence interval, defining the minimum and maximum.
A confidence interval, parameterized by 
$\alpha'\in(0,1)$, is an interval satisfying that the probability of an unknown sample being inside it is at least $1-\alpha'$.
It is defined by: $[l_y,u_y]=\left[F_y^{-1}\left(\frac{\alpha'}{2}\right),F_y^{-1}\left(1-\frac{\alpha'}{2}\right)\right]$.
Since we wish to compute an interval hyper-network $\mathcal{N}^\#$ abstracting an unseen network with probability $1-\alpha$ and since there are $|\mathcal{W}\cup \mathcal{B}|$ parameters, we choose $\alpha'=\frac{\alpha}{|\mathcal{W}\cup \mathcal{B}|}$ (\Cref{ln:abend}).
By the union bound, we obtain confidence intervals guaranteeing that the probability that an unseen concrete network is covered by the hyper-network is at least $1-\alpha$.

\paragraph{Stopping condition}
The goal of \alg's stopping condition is to identify when the distributions of the intervals computed for $\mathcal{N}^\#_{curr}$ have converged to their expected distributions.
This is the case when the intervals have not changed significantly in $\mathcal{N}^\#_{curr}$ compared to $\mathcal{N}^\#_{prev}$.
Formally, for each network parameter $w$, it compares the current interval to the previous interval by computing their Jaccard distance.
Given the previous and current intervals of $w$, $I_{prev}$ and $I_{curr}$, the Jaccard Index is:
$J(I_{curr},I_{prev})=\frac{I_{curr}\sqcap I_{prev}}{I_{curr}\sqcup I_{prev}}$.
For any two intervals, the Jaccard distance $1-J(I_{curr},I_{prev})\in[0,1]$, such that the smaller the distance, the more similar the intervals are.
If the Jaccard distance is below a ratio $R$ (a small number), we consider the interval of $w$ as converged to the expected CDF.
If at least $M\cdot100\%$ of the hyper-network's intervals have converged, we consider that the hyper-network has converged, and thus \alg terminates.

%% file: alg_flow.tex
\tikzstyle{empty} = [rectangle, minimum width=0cm, minimum height=1cm, text centered, draw=white, fill=white]
\tikzstyle{block} = [rectangle, minimum width=2cm, minimum height=1cm, text centered, text width=2cm, draw=black, fill=white]
\tikzstyle{wblock} = [rectangle, minimum width=3cm, minimum height=1cm, text centered, text width=3cm, draw=black, fill=white]
\tikzstyle{arrow} = [thick,->,>=latex]

\begin{tikzpicture}[node distance=1.7cm]

\node (ws) [empty] { };
\node (norm) [block, right of=ws, xshift=2cm] {normalize};
\node (trans) [block, below of=norm] {transform};
\node (hp) [wblock, below of=trans, xshift=2.5cm] {compute a confidence interval with probability $1-\alpha'$};
\node (invt) [block, above of=hp, xshift=2.5cm] {$\text{transform}^{-1}$};
\node (invn) [block, above of=invt] {$\text{normalize}^{-1}$};
\node (wb) [empty, right of=invn, xshift=2cm] { };

\draw [arrow] (ws) -- node[anchor=south] {$\{w_1,\ldots,w_K\}$} (norm);
\draw [arrow] (norm) -- node[anchor=east] {$\{z_1,\ldots,z_K\}$} (trans);
\draw [arrow] (trans) |- node[anchor=south east]{$\{y_1,\ldots,y_K\}$} (hp);
\draw [arrow] (hp) -| node[anchor=south west]{$[l_y,u_y]$} (invt);
\draw [arrow] (invt) -- node[anchor=west] {$[l_z,u_z]$} (invn);
\draw [arrow] (invn) -- node[anchor=south] {$[l_w,u_w]$} (wb);

\end{tikzpicture}

%% file: verify.tex
\subsection{Verification of a Hyper-Network}
\label{sec:verify}

In this section, we explain how \tool verifies that an interval hyper-network is locally robust, in order to show that the network is \prop.
Our verification extends a local robustness verifier~\cite{robust_milp}, designed for checking local robustness of a (concrete) network, to check local robustness given an interval hyper-network.
This verifier encodes the verification task as a mixed-integer linear program (MILP), which is submitted to a MILP solver, and provides a sound and complete analysis. We note that we experimented with extending incomplete analysis techniques to our setting (interval analysis and DeepPoly~\cite{abst_cert}) to obtain a faster analysis. However, the over-approximation error stemming both from these techniques and our hyper-network led to a low precision rate.
The challenge with extending the verifier by Tjeng et al.~\cite{robust_milp} is that both the neuron values and the network parameters are variables, leading to quadratic constraints, which are computationally heavy for constraint solvers.
Instead, \tool relaxes the quadratic constraints.
To mitigate an over-approximation error, we propose two approaches.
We begin with a short background and then describe our extension.

\paragraph{Background}
The MILP encoding by Tjeng et al.~\cite{robust_milp} precisely captures the neurons' affine computation as linear constraints.
The ReLU computations are more involved, because ReLU is non-linear.
For each ReLU computation, a boolean variable is introduced along with four linear constraints.
The neighborhood is expressed by constraining each input value by an interval, and the local robustness check is expressed by a linear constraint over the network's output neurons.
This encoding has been shown to be sound and complete.
To scale the analysis, every neuron is associated with a real-valued interval.
This allows to identify ReLU neurons whose computation is linear, which is the case if the input's interval is either non-negative or non-positive.
In this case, the boolean variable is not introduced for this ReLU computation and the encoding's complexity decreases.

\paragraph{Extension to hyper-networks}
To extend this verifier to analyze hyper-networks, we encode the affine computations differently because network parameters are associated with intervals and not real numbers.
A naive extension replaces the parameter values in the linear constraints by variables, which are bounded by intervals.
However, this leads to quadratic constraints and significantly increases the problem's complexity.
To keep the constraints linear, one option is to introduce a fresh variable for every multiplication of a neuron's input and a network parameter.
However, such variable would be bounded by the interval abstracting the multiplication of the two variables, which may lead to a very large over-approximation error.
Instead, we rely on the following observation: if the input to every affine computation is non-negative, then the abstraction of the multiplication does not introduce an over-approximation error.
This allows us to replace the constraint of each affine variable $\hat{x}$ (defined in~\Cref{sec:prelim}), previously captured by an equality constraint, with two inequalities providing a lower and upper bound on $\hat{x}$.
Formally, given lower and upper bounds of the weights and biases at the matrices $l_W$ and $u_W$ and the vectors $l_b$ and $u_b$, the affine variable $\hat{x}$ is bounded by:
$$l_W\cdot x+l_b\leq\hat{x}\leq u_W\cdot x+u_b$$
To guarantee that the input to every affine computation $x$ is non-negative our analysis requires (1)~preprocessing of the network's inputs and (2)~a non-negative lower bound to every activation function's output.
The second requirement holds since the MILP encoding of ReLU explicitly requires a non-negative output.
If preprocessing is inapplicable (i.e., $x\in\mathcal{X}\nsubseteq[0,1]^d$) or the activation function may be negative (e.g., leaky ReLU), we propose another approach to mitigate the precision loss, given a lower bound $l_x\leq x$.
Given lower and upper bounds for the weights and biases $l_W$, $u_W$, $l_b$, and $u_b$, we can bound the output by:
$$l_W\cdot x+l_b-(u_W-l_W)\cdot\max(0,-l_x)\leq\hat{x}\leq u_W\cdot x+u_b+(u_W-l_W)\cdot\max(0,-l_x)$$
We provide a proof in 
\ifthenelse{\boolean{shortver}}{the extended version~\cite{}.}{\Cref{sec:proofs}.}
Note that each bound is penalized by $(u_W-l_W)\cdot\max(0,-l_x)\geq0$, which is an over-approximation error term for the case where the lower bound $l_x$ is negative. 


\subsection{Analysis of Sphynx}
\label{sec:algcorr}
In this section, we discuss the correctness of \tool and its running time. Proofs are provided in 
\ifthenelse{\boolean{shortver}}{the extended version~\cite{}.}{\Cref{sec:proofs}.}

\paragraph{Correctness}
Our first lemma states the conditions under which \intalg computes an abstraction for the values of a single network parameter with a high probability. 
\begin{restatable}[]{lemma}{ftc}
\label{lem}
Given a parameter $w$ and an error bound $\alpha'$, if the observed values $w_1,\ldots,w_K$ are IID and suffice to predict the correct distribution, 
and if there exists $\lambda\in[0,2]$ such that the distribution of the transformed normalized values $y_1,\ldots,y_K$ is similar to a distribution given by KDE with a Laplace kernel \ifthenelse{\boolean{shortver}}{}{(using the bandwidth defined in~\Cref{app:kde})}, then \intalg computes a confidence interval containing an unseen value $w_i$, for $i\in\{K+1,\ldots,|\mathcal{D}|+1\}$, 
with a probability of $1-\alpha'$.
 \end{restatable}   

Note that our lemma does not make any assumption about the distribution of the observed values $w_1,\ldots,w_K$.
Next, we state our theorem pertaining the correctness of \alg. The theorem states that when 
the stopping condition identifies correctly when the observed values have converged to the correct distribution, then the hyper-network abstracts an unseen network with the expected probability.
\begin{restatable}[]{theorem}{fta}
  Given a network $N$, its training set $\mathcal{D}$, its training algorithm $\mathcal{T}$, and an error bound $\alpha$, if 
  $R$ is close to $0$ and $M$ is close to $1$,
    then \alg returns a hyper-network abstracting an unseen network with probability $1-\alpha$.
\end{restatable}

%

Our next theorem states that our verifier is sound and states when it is complete.
Completeness means that if the hyper-network is locally robust at the given neighborhood, \tool is able to prove it.
\begin{restatable}[]{theorem}{ftb}
\label{thm2} 
Our extension to the MILP verifier provides a sound analysis.
It is also complete if all inputs to the affine computations are non-negative.
\end{restatable}
By \Cref{thm2} and \Cref{def:ldcp}, if \tool determines that the hyper-network is locally robust, then the network is \prop.
Note that it may happen that the network is \prop, but the hyper-network is not locally robust due to the abstraction's precision loss.

\paragraph{Running time}
The running time of \tool consists of the running time of \alg and the running time of the verifier.
The running time of \alg consists of the networks' training time and the time to compute the hyper-networks (the running time of the stopping condition is negligible).
The training time is the product of the number of networks and the execution time of the training algorithm $\mathcal{T}$.
The time complexity of \intalg is $O(K^2)$, where $K=|\texttt{nets}|$, and thus the computation of a hyper-network is: $O\left(K^2\cdot|\mathcal{W}\cup\mathcal{B}|\right)$.
Since \alg runs $\frac{K}{k}$ iterations, overall, the running time is $O\left(|\mathcal{T}|\cdot K + \frac{K^3}{k}\cdot|\mathcal{W}\cup\mathcal{B}|\right)$.
In practice, the second term is negligible compared to the first term ($|\mathcal{T}|\cdot K$).
Thus, the fewer trained networks the faster \alg is.
The running time of the verifier is similar to the running time of the MILP verifier~\cite{robust_milp}, verifying local robustness of a single network, which is exponential in the number of ReLU neurons
whose computation is non-linear (their input's interval contains negative and positive numbers).
Namely, \tool reduces the verification time by a factor of $|\mathcal{D}|+1$ compared to the naive algorithm that verifies robustness network-by-network.

%% file: eval.tex
\section{Evaluation}
\label{sec:eval}


We implemented \tool in Python\footnote{Code is at: \url{https://github.com/Robgy/Verification-of-Neural-Networks-Privacy}}.
Experiments ran on an Ubuntu 20.04 OS on a dual AMD EPYC 7713 server with 2TB RAM and $8$ NVIDIA A100 GPUs.
The hyper-parameters of \alg are: $\alpha=0.1$, the number of trained networks in every iteration is $k=400$, the thresholds of the stopping condition are $M=0.9$ and $R=0.1$.
We evaluate \tool on the three data-sensitive datasets described in~\Cref{sec:prelim}: Adult Census~\cite{uci} (Adult), Bank Marketing~\cite{bank_market} (Bank), and Default of Credit Card Clients~\cite{credit_card} (Credit).
We preprocessed the input values to range over $[0,1]$ as follows.
Continuous attributes were normalized to range over $[0,1]$ and categorical attributes were transformed into two features ranging over $[0,1]$: $\cos{\left(\frac{\pi}{2}\cdot\frac{i}{m-1}\right)}$ and $\sin{\left(\frac{\pi}{2}\cdot\frac{i}{m-1}\right)}$, where $m$ is the number of categories and $i$ is the category's index $i\in\{0,\ldots,m-1\}$.
Binary attributes were transformed with a single feature: 0 for the first category and 1 for the second one.
While one hot encoding is highly popular for categorical attributes, it has also been linked to reduced model accuracy when the number of categories is high~\cite{onehot,RodriguezBGE18}. However, we note that the encoding is orthogonal to our algorithm.
We consider three fully-connected networks: $2\times50$, $2\times100$, and $4\times50$, where the first number is the number of intermediate layers and the second one is the number of neurons in each intermediate layer.
Our network sizes are comparable to or larger than the sizes of the networks analyzed by verifiers targeting these datasets~\cite{paral_fair,redu_fair,BaharloueiNBR20}.
All networks reached around $80\%$ accuracy.
The networks were trained over 10 epochs, using SGD with a batch size of 1024 and a learning rate of 0.1.
We used $L_1$ regularization with a coefficient of $10^{-5}$.
All networks were trained with the same random seed, so \tool can identify the maximal privacy leakage (allowing different seeds may reduce the privacy leakage since it can be viewed as adding noise to the training process).
We remind that \tool's challenge is intensified compared to local robustness verifiers: 
their goal is to prove robustness of a single network, whereas \tool's goal is to prove that privacy is preserved over a very large number of concrete networks: 32562 networks for Adult, 31649 networks for Bank and 21001 networks for Credit.
Namely, the number of parameters that \tool reasons about is 
the number of parameters of all $|\mathcal{D}|+1$ concrete networks.
Every experiment is repeated 100 times, for every network, where each experiment randomly chooses dataset entries and trains their respective networks.

\paragraph{Performance of \tool}
We begin by evaluating \tool's ability to verify \prop.
We consider three kinds of neighborhoods, each is defined given an input $x$, and the goal is to prove that all inputs in the neighborhood are classified as a label $y$:
\begin{enumerate}[nosep,nolistsep]
  \item \emph{Membership}, $I(x)=\left\{x\right\}$: safety to label-only membership attacks~\cite{label_mia,member_label,Monerale23}.
  \item \emph{DP-Robustness}, $I^\infty_\epsilon(x)=\left\{x'\mid\|x'-x\|_\infty \leq \epsilon\right\}$, where $\epsilon=0.05$: 
 safety to adversarial example attacks in a DP setting (similarly to~\cite{LecuyerAG0J19,PhanTHJSD20}).
  \item \emph{Sensitivity}, $I^S_i(x)=\left\{x'\mid\bigwedge_{j\in [d]\setminus\{i\}} x'_j=x_j\right\}$, where the sensitive feature is $i=$\emph{sex} for Adult and Credit and $i=$\emph{age} for Bank: safety to attacks revealing sensitive attributes (like~\cite{model_inv,FredriksonLJLPR14}). We note that sensitivity is also known as \emph{individual fairness}~\cite{verify_fair}.
\end{enumerate} 
For each dataset, we pick $100$ inputs for each of these neighborhoods.
We compare \tool to the naive but most accurate algorithm that trains all concrete networks and verifies the neighborhoods' robustness network-by-network.
Its verifier is the MILP verifier~\cite{robust_milp} on which 
\tool's verifier builds.
We let both algorithms run on all networks and neighborhoods.
\Cref{tab:rq3-new} reports the confusion matrix of \tool compared to the ground truth (computed by the naive algorithm): 
\begin{itemize}[nosep,nolistsep]
  \item True Positive (TP): the number of neighborhoods that are \prop and that \tool returns they are \prop.
 \item True Negative (TN): the number of neighborhoods that are not \prop and \tool returns they are not \prop.
 \item False Positive (FP): the number of neighborhoods that are not \prop and \tool returns they are \prop.
 A false positive may happen because of the probabilistic abstraction which may miss concrete networks that are not locally robust at the given neighborhood.
 \item False Negative (FN): the number of neighborhoods that are \prop but \tool returns they are not \prop.
 A false negative may happen because the hyper-network may abstract spurious networks that are not locally robust at the given neighborhood (an over-approximation error).
\end{itemize}
Results show that \tool's average accuracy is 93.3\% (TP+TN).
The FP rate is 1.1\% and at most 9\%.
The FN rate (i.e., the over-approximation error) is 5.5\%.
Results further show how private the different networks are.
All networks are very safe to label-only membership attacks.
Although \tool has several false negative results, it still allows the user to infer that the networks are very safe to such attack.
For DP-Robustness,
results show that some networks are fairly robust (e.g., Bank $2\times 50$ and $2\times 100$), while others are not (e.g., Bank $4\times 50$).
For Sensitivity, results show that \tool enables the user to infer what networks are sensitive to the sex/age attribute (e.g., Credit $4\times 50$) and what networks are not (e.g., Credit $2\times 100$).
An important characteristic of \tool's accuracy is that the false positive and false negative rates do not lead to inferring a wrong conclusion.
For example, if the network is DP-robust, \tool proves \prop (TP+FP) for significantly more DP-robustness neighborhoods than the number of DP-robustness neighborhoods for which it does not prove \prop (TN+FN).
Similarly, if the network is not DP-robust, \tool determines for significantly more DP-robustness neighborhoods that they are not \prop (TN+FN) than the number of DP-robustness neighborhoods that it proves they are \prop (TP+FP).

\begin{table}[t]
\small
\begin{center}
\caption{\tool's confusion matrix.}
\begin{tabular}{l@{\hspace{0.2cm}}c@{\hspace{0.4cm}}r@{\hspace{0.2cm}}r@{\hspace{0.2cm}}r@{\hspace{0.2cm}}r@{\hspace{0.4cm}}r@{\hspace{0.2cm}}r@{\hspace{0.2cm}}r@{\hspace{0.2cm}}r@{\hspace{0.4cm}}r@{\hspace{0.2cm}}r@{\hspace{0.2cm}}r@{\hspace{0.2cm}}r}
\toprule
Dataset & Network & \multicolumn{4}{c}{Membership} & \multicolumn{4}{c}{DP-Robustness} & \multicolumn{4}{c}{Sensitivity} \\

      &              &   TP & TN   &     FP &    FN & TP    & TN   & FP       & FN    & TP   & TN  & FP    & FN  \\
\midrule
Adult & $2\times50$  &  93    & 0    &  0     & 7  &  75   & 21    & 0    & 4    &  85      & 10  & 0   &  5  \\
 & $2\times100$      &  82    & 1    & 0      & 17 &  54   & 39    & 0    & 7    &  75      & 10  & 0   & 15  \\
 & $4\times50$       &  93    & 0    & 0      & 7  &  11   & 86    & 3    & 0    &   1      & 97  & 1   & 1   \\
\midrule 

Bank & $2\times50$   &  100   & 0    & 0     & 0     & 100   & 0      & 0 & 0    & 100   & 0     &  0     & 0  \\
 & $2\times100$      &  99    & 0    & 0     & 1     & 98    &  1     & 0 & 1    & 99    & 0     & 0      & 1 \\
 & $4\times50$       &  81    & 0    & 0     & 19    & 22    &  62    & 9 & 7    & 2     & 71    & 8      & 19 \\
\midrule

Credit & $2\times50$ &  93    & 0   & 0      & 7   &  91    & 0     & 0     & 9    & 92   & 0   &  0     & 8  \\
 & $2\times100$      &  100   & 0   & 0      & 0   &  91    & 2     & 2     & 5    & 100  & 0   &  0     & 0  \\
 & $4\times50$       &  91    & 0   & 0      & 9   &  2     & 95    & 3     & 0    & 0    & 96   &   4  & 0  \\
\bottomrule
\end{tabular}
    \label{tab:rq3-new}
    \end{center}
\end{table}

%
\Cref{tab:traintime} compares the execution time of \tool to the naive algorithm.
It shows the number of trained networks (which is the size of the dataset plus one for the naive algorithm, and $K=|nets|$ for \tool), the overall training time on a single GPU in hours and the verification time of a single neighborhood (in hours for the naive algorithm, and in \emph{seconds} for \tool).
Results show the two strengths of \tool: (1)~it reduces the training time by 13.6x, because it requires to train only 7\% of the networks and (2)~it reduces the verification time by $1.7\cdot10^4$x.
Namely, \tool reduces the execution time by four orders of magnitude compared to the naive algorithm.
The cost is the minor decrease in \tool's accuracy ($<$7\%).
That is, \tool trades off precision with scalability, like many local robustness verifiers do~\cite{veri_nlin,opt_abst,form_safe,bound_prop,cnn_cert,safe_abst,abst_cert,convex_verify,beyond_single,poly_gpu}.

\begin{table}[t]
\small
\begin{center}
\caption{Training and verification time of \tool and the naive algorithm.}
\begin{tabular}{l@{\hspace{0.2cm}}l@{\hspace{0.4cm}}c@{\hspace{0.3cm}}c@{\hspace{0.5cm}}c@{\hspace{0.4cm}}c@{\hspace{0.8cm}}c@{\hspace{0.3cm}}c@{\hspace{0.5cm}}}
\toprule
Dataset & Network &  \multicolumn{2}{l}{Trained networks} & \multicolumn{2}{l}{GPU training time}& \multicolumn{2}{l}{Verification time}\\
      &               & naive& \tool  &  naive & \tool & naive  & \tool  \\
      &               & $|\mathcal{D}|+1$ & $K$  & hours &  hours   &       hours & seconds \\

\midrule
Adult & $2\times50$   &   32562& 2436 & 44.64     &     3.33      &  2.73 & 0.35\\
       & $2\times100$  &    32562 & 2024 &  46.56     &     2.89     & 5.24 & 0.72 \\
      & $4\times50$   &    32562& 1464 & 52.37     &     2.35      &  0.33 &  0.87\\
\midrule
Bank  & $2\times50$   &     31649& 2364  & 41.04     &     3.06      & 2.73& 0.35\\
       & $2\times100$  & 31649 & 2536 &    41.76     &     3.34      & 5.60&  0.69\\
      & $4\times50$    & 31649& 1996 &    49.41     &     3.11       & 1.3 & 1.19\\
\midrule
Credit & $2\times50$  &   21001& 2724 & 18.72     &     2.42      &  2.1& 0.35\\
          & $2\times100$ &   21001  & 2234  & 19.21     &     2.04    & 3.6& 0.64 \\
       & $4\times50$  &    21001& 1816 &22.67     &     1.96     &   0.08 &0.75 \\
\bottomrule
\end{tabular}
    \label{tab:traintime}
    \end{center}
\end{table}


\begin{table}[t]
\small
\begin{center}
\caption{\intalg vs. the interval abstraction and several variants.}
\begin{tabular}{l@{\hspace{0.2cm}}c@{\hspace{0.2cm}}c@{\hspace{0.2cm}}c@{\hspace{0.2cm}}c@{\hspace{0.2cm}}c}
\toprule
 & Int. Abs. & -Transform & -Normalize & -KDE & \intalg \\
\midrule
Weight abstraction rate & 19.60 & 55.25 & 22.89 & 48.10 & 70.61 \\
\midrule
Miscoverage & 7.20 & 5.42 & 1.5$\times 10^{-6}$ & 2.13 &  2.49 \\
\midrule
Overcoverage & 1 & 1.12 & 299539 & 1.62 & 2.80 \\
\bottomrule
\end{tabular}
    \label{tab:rq1}
    \end{center}
\end{table}

\begin{table}[t]
\small
\begin{center}
\caption{\alg vs. the interval abstraction variant (using $K$ networks).}
\begin{tabular}{l@{\hspace{0.2cm}}c@{\hspace{0.4cm}}r@{\hspace{0.2cm}}r@{\hspace{0.4cm}}r@{\hspace{0.2cm}}r}
\toprule
Dataset & Network & \multicolumn{2}{c}{Network abstraction rate} & \multicolumn{2}{c}{Overcoverage}  \\
 & & Int. Abs. & \tool & Int. Abs. & \tool \\
\midrule
Adult & $2\times50$ &   77.04 &  94.55  & 1.00 &   4.52  \\
      & $2\times100$ &   55.20  &   92.82      & 1.00 &    2.80     \\
      & $4\times50$  &   58.24 &  92.22  & 1.00 &   2.83  \\
\midrule
Bank & $2\times50$  &   80.06 & 93.89  & 1.00 & 2.54   \\
     & $2\times100$ &   73.21 & 90.13       & 1.00 &   3.21      \\
     & $4\times50$  &   74.80 & 90.47  & 1.00 & 1.93   \\
\midrule
Credit & $2\times50$ &  84.07 & 95.98  & 1.00 & 15.07  \\
      & $2\times100$ &  67.16 & 90.59   & 1.00 &  3.68      \\
      & $4\times50$  &  73.14 & 93.19   & 1.00 & 3.53   \\
\bottomrule
\end{tabular}
    \label{tab:rq2}
    \end{center}
\end{table}

\paragraph{Ablation study}
We next study the importance of \tool's steps in predicting an interval hyper-network.
Recall that \intalg predicts an interval for a given network parameter by running a normalization and the Yeo-Johnson transformation and transforming it into a distribution given by KDE.
We compare \intalg to interval abstraction, mapping a set of values into the interval defined by their minimum and maximum, and to three variants of \intalg:
(1)~\emph{-Transform}: does not use the normalization or the transformation and directly estimates the density with KDE, (2)~\emph{-Normalize}: does not use the normalization but uses the transformation,
(3)~\emph{-KDE}: transforms into a Gaussian distribution (as common), and thus employs a normalization based on the $L_2$ norm.
We run all approaches on the $2\times100$ network, trained for Adult.
\Cref{tab:rq1} reports the following:
\begin{itemize}[nosep,nolistsep]
  \item Weight abstraction rate: the average percentage of weights whose interval provides a (sound) abstraction.
  Note that this metric is not related to \Cref{lem}, guaranteeing that the value of a network parameter of a \emph{single} network is inside its predicted interval with probability $1-\alpha'$. 
  This metric is more challenging: it measures how many values of a given network parameter, over $|\mathcal{D}|+1$ networks, are inside the corresponding predicted interval.
  \item Miscoverage:  measures how much the predicted intervals need to expand to be an abstraction.
  It is the average over all intervals' miscoverage.
  The interval miscoverage is the ratio of the size of the difference between the optimal interval and the predicted interval and the size of the predicted interval.
  \item {Overcoverage}: measures how much wider than necessary the predicted intervals are.
  It is the geometric average over all intervals' overcoverage.
  The interval overcoverage is the ratio of the size of the join of the predicted interval and the optimal interval and the size of the optimal interval.
\end{itemize}
Results show that the weight abstraction rate of the interval abstraction is very low and \intalg has a 3.6x higher rate.
Results further show that \intalg obtains a very low miscoverage, by less than 2.9x compared to the interval abstraction.
As expected, these results come with a cost: a higher overcoverage.
An exception is the interval abstraction which, by definition, does not have an overcoverage.
Results further show that the combination of normalization, transformation, and KDE improve the weight abstraction rate of \intalg.
Next, we study how well \alg's hyper-networks abstract an (unseen) concrete network with a probability $\geq0.9$.
We compare to a variant that replaces \intalg by the interval abstraction, computed using the $K$ concrete networks reported  in~\Cref{tab:traintime}.
\Cref{tab:rq2} reports the network abstraction rate (the average percentage of concrete networks abstracted by the hyper-network) and overcoverage.
Results show that \alg obtains a very high network abstraction rate and always above $1-\alpha=0.9$.
In contrast, the variant obtains lower network abstraction rate with a very large variance.
As before, the cost is the over-approximation error.

%% file: related.tex
\section{Related Work}
\label{sec:related}



\paragraph{Network abstraction}
Our key idea is to abstract a set of concrete networks (seen and unseen) into an interval hyper-network.
Several works rely on network abstraction to expedite verification of a single network.
The goal of the abstraction is to generate a smaller network, which can be analyzed faster, and that preserves soundness with respect to the original network.
One work proposes an abstraction-refinement approach that abstracts a network by splitting neurons into four types and then merging neurons of the same type~\cite{abst_net}.
Another work merges neurons similarly but chooses the neurons to merge by clustering~\cite{deep_abst}.
Other works abstract a neural network by abstracting its network parameters with intervals~\cite{PrabhakarA19} or other abstract domains~\cite{SotoudehT20}. 
In contrast, \tool abstracts a large set of networks, seen and unseen, and proves robustness for all of them.

\paragraph{Robustness verification}
\tool extends a MILP verifier~\cite{robust_milp} to verify local robustness given a hyper-network.
There are many local robustness verifiers.
Existing verifiers leverage various techniques, e.g., over-approximation~\cite{veri_nlin,opt_abst,form_safe}, linear relaxation~\cite{bound_prop,cnn_cert,safe_abst,abst_cert,convex_verify,beyond_single,poly_gpu}, simplex~\cite{smt_relu,marab_verify,abst_net}, mixed-integer linear programming~\cite{robust_milp,milp_bin,boost_cert}, and duality~\cite{dual_scale,cert_def}.
A different line of works verifies robustness to small perturbations to the network parameters~\cite{WengZ0CLD20,TsaiHYC21}. These works assume the parameters' perturbations are confined in a small $L_\infty$ $\epsilon$-ball and compute a lower and upper bounds on the network parameters by linearly bounding their computations. In contrast, our network parameters are not confined in an $\epsilon$-ball, and our analysis is complete if the inputs are (or processed to be) non-negative. 
\paragraph{Adaptive data analysis}
\alg relies on an iterative algorithm to predict the minimum and maximum of every network parameter.
Adaptive data analysis deals with estimating statistics based on data that is obtained iteratively.
Existing works focus on statistical queries computing the expectation of functions~\cite{valid_adapt,prevent_false,general_adapt,pres_adapt,adapt_median,emp_adapt,limit_general} or low-sensitivity queries~\cite{stab_adapt}.

%% file: conc.tex
\section{Conclusion}
\label{sec:conc}
We propose a privacy property for neural networks, called \propl, extending local robustness to the setting of differential privacy for black-box classifiers.
We then present \tool, a verifier for determining whether a network is \prop at a given neighborhood.
Instead of training all networks and verifying local robustness network-by-network, \tool predicts an interval hyper-network, providing an abstraction with a high probability, from a small number of networks.
To predict the intervals, \tool transforms the observed parameter values into a distribution given by KDE, using the Yeo-Johnson transformation.
\tool then verifies \prop at a neighborhood directly on the hyper-network, by extending a local robustness MILP verifier.
To mitigate an over-approximation error, we rely on preprocessing to the network's inputs or on a lower bound for them.
We evaluate \tool on data-sensitive datasets and show that  
by training only 7\% of the networks, \tool predicts a hyper-network abstracting any concrete network with a probability of at least $0.9$,
obtaining $93\%$ verification accuracy and 
reducing the verification time by $1.7\cdot10^4$x.

%% file: kde.tex
\section{Kernel Density Estimation}
\label{app:kde}
\sloppy
As described, \intalg adapts the Yeo-Johnson transformation to transform an unknown distribution into a distribution given by kernel density estimation (KDE) with a Laplace kernel.
We next provide a short background and explain how \intalg uses KDE.
KDE is a kernel smoothing technique, estimating an unknown distribution using a weighted average of neighboring samples.
The weight is determined by a \emph{kernel}, which is a probability density function $g(\cdot)$ with mean 0.
KDE is parameterized by a positive \emph{bandwidth} $h$, determining the scale of the kernel.
The density estimation function is: $\hat{f}_y(v)=\frac{1}{Kh}\sum_{i=1}^{K}g\left(\frac{v-y_i}{h}\right)$.
The advantages of KDE is that it can capture complex densities, and that it converges to the true density when $K\to\infty$ and $h\to0$.
Many popular kernel functions are bounded, making them unsuitable for our setting, where the goal is to predict the minimum and maximum values.
We thus focus on unbounded kernels.
The Gaussian kernel is unbounded, however it is challenging for computing the confidence interval since the decaying rate varies.
We thus rely on the Laplace kernel $g(x)=\frac{1}{2}e^{-|x|}$, also known as the exponential kernel.
It is especially suitable for our setting because it has a heavy tail, resulting in a wider interval, leaving room for unseen outliers.
Additionally, computing its confidence interval can be done efficiently.
Technically, the confidence interval is: $[l_y,u_y]=\left[F_y^{-1}\left(\frac{\alpha'}{2}\right),F_y^{-1}\left(1-\frac{\alpha'}{2}\right)\right]=\left[-h\log\left(\frac{1}{K}\sum_{i=1}^Ke^{-\frac{y_i}{h}}\right)-h\cdot\log\left(\frac{1}{\alpha'}\right),h\log\left(\frac{1}{K}\sum_{i=1}^Ke^{\frac{y_i}{h}}\right)+h\cdot\log\left(\frac{1}{\alpha'}\right)\right]$,
where $F$ is the CDF of the kernel density estimator and $y_i$-s are the transformed values.
Note that this is a non-symmetric interval which is wider than the interval that we would obtain if we assumed the Laplace distribution (without KDE), in which case the interval is symmetric and equal to: $\left[\text{median}(y_1\ldots,y_K)-h\cdot\log\left(\frac{1}{\alpha'}\right),\text{median}(y_1\ldots,y_K)+h\cdot\log\left(\frac{1}{\alpha'}\right)\right]$.
The bandwidth $h$ is a hyper-parameter, but it cannot be computed using MLE because as $h$ decreases, the likelihood increases, not considering any neighboring samples.
Instead, the bandwidth is usually chosen to minimize the error between the real (unknown) distribution and a distribution given by KDE.
However, this results in a too small bandwidth leading to a too small interval, which will not abstract well.
On the other hand, choosing a too large bandwidth leads to a too large over-approximation error.
To balance, \intalg sets $h$ to be the centered absolute first moment: $h=\frac{1}{K}\sum_{i=1}^K|y_i-\mu_y|$, where $\mu_y=\text{median}\{y_1,\ldots,y_K\}$.

%% file: proofs.tex
\section{Proofs}\label{sec:proofs}
In this section, we provide proofs.

\ftc*
\begin{proof}
Since the normalization and transformation are bijective functions, for $[l_w,u_w]$ to contain an unseen value $w_i$, for $i\in\{K+1,\ldots,|\mathcal{D}|+1\}$, with a probability of $1-\alpha'$, it suffices to show that $[l_y,u_y]$ satisfies this requirement for $y_i$.
  For $[l_y,u_y]$ to satisfy this requirement, by the computation of a confidence interval, three conditions have to hold (1)~the transformed values $y_1,\ldots,y_K$ have to be IID, (2)~they have to be sufficient to predict the correct distribution, and (3)~\intalg has to have their distribution. 
Because the normalization and transformation are bijective functions, conditions (1) and (2) stem from the assumption that $w_1,\ldots,w_K$ satisfy conditions (1) and (2). 
To have the distribution of $y_1,\ldots,y_K$ (condition (3)), \intalg requires the Yeo-Johnson transformation to succeed.  
While there is no guarantee, similarly to~\cite{transform}, this transformation succeeds if its inputs $z_1,\ldots,z_K$ are IID, sufficient to predict the correct distribution and there exists $\lambda\in[0,2]$ such that the distribution of the transformed normalized values $y_1,\ldots,y_K$ is \emph{similar} to a distribution given by KDE with a Laplace kernel. 
The inputs $z_1,\ldots,z_K$ are practically IID because $w_1,\ldots,w_K$ are IID and the normalization is a bijective function  
(the only source of dependence is that given $K-1$ normalized samples, the last normalized sample can be inferred). The inputs are sufficient to predict the correct distribution because we assume $w_1,\ldots,w_K$ are sufficient. Lastly, the requirement about $\lambda$ is one of the lemma's assumptions.
\end{proof}

\fta*

\begin{proof}
We show that the conditions of \Cref{lem} hold, for every weight $w$ and for $\alpha'=\frac{\alpha}{|\mathcal{W}\cup \mathcal{B}|}$, and thus by the union bound, it follows that the hyper-network provides an abstraction with a probability of $1-\alpha$. Let $w$ be a network parameter.
First, 
the observed values $w_1,\ldots,w_K$ are IID since each is computed independently by the training algorithm $\mathcal{T}$ when given the training set without a random sample. 
Second, the number of observed values is sufficient to predict the correct distribution, because of our stopping condition.
By its definition, if the stopping condition is true, $R$ is close to $0$, and $M$ is close to $1$, then all weights' distributions have converged to their expected distribution. 
Third, the lemma requires that there exists $\lambda\in[0,2]$ such that the distribution of the transformed normalized values $y_1,\ldots,y_K$ is \emph{similar} to a distribution given by KDE with a Laplace kernel (using the bandwidth defined in~\Cref{app:kde}). This holds in practice because, given IID samples that suffice to predict the correct distribution, KDE provides a good estimation for an unknown distribution and constraining $\lambda\in[0,2]$ practically does not affect this estimation. 
\end{proof}

\ftb*
\begin{proof}
Consider an affine variable $\hat{x}=W\cdot x + b$.
Assume its input $x$ is non-negative.
Given the lower and upper bounds on the weights and biases, because $x$ is non-negative and by interval arithmetic, $\hat{x}$ is bounded in $[l_W\cdot x+l_b, u_W\cdot x+u_b]$. Moreover, this interval is tight, since the lower bound is obtained when $W=l_W$ and $b=l_b$ and the upper bound is obtained when $W=u_W$ and $b=u_b$. 
The ReLU encoding is identical to~\cite{robust_milp} and thus soundly and precisely captures its computation. Similarly, the neighborhood's encoding and the local robustness check's encoding are identical to~\cite{robust_milp} and are thus sound and complete.
 If $x$ may be negative but has a lower bound $l_x\leq x$, then our encoding bounds $\hat{x}$ in $[l_W\cdot x+l_b-(u_W-l_W)\cdot\max(0,-l_x),u_W\cdot x+u_b+(u_W-l_W)\cdot\max(0,-l_x)]$. This bound is sound because we can write $\hat{x}=W\cdot x+b=W\cdot(x-l_x)+W\cdot l_x+b$. Then, we rewrite this expression as follows: 
 $W\cdot(x-l_x)+W\cdot \max\{0,l_x\}-W\cdot \max\{0,-l_x\}+b$. Note that $x-l_x\geq 0$, $\max\{0,l_x\}\geq0$, and $-\max\{0,-l_x\}\leq 0$. Thus, this expression is upper bounded by:
 $u_W\cdot(x-l_x)+u_W\cdot \max\{0,l_x\}-l_W\cdot \max\{0,-l_x\}+u_b$. By rearranging it, we obtain the upper bound:
  $u_W\cdot x +u_b + u_W\cdot(\max\{0,l_x\}-l_x)-l_W\cdot \max\{0,-l_x\}=u_W\cdot x +u_b + (u_W-l_W)\cdot \max\{0,-l_x\}$. Similarly, we obtain the lower bound.

\end{proof}